\documentclass[12pt]{article}
\usepackage{tikz}
\usetikzlibrary{arrows}
\usepackage{latexsym,amsmath,amsfonts,amsthm}
\usepackage{graphicx}
\usepackage{graphics}
\usepackage{mathrsfs}
\usepackage{subfigure}
\usepackage{pstricks, pst-node}
\usepackage[lined,boxed,commentsnumbered]{algorithm2e}
\usetikzlibrary{matrix}

\newtheorem{theorem}{Theorem}
\newtheorem{proposition}[theorem]{Proposition}

\newtheorem{corollary}[theorem]{Corollary}

\def\bR{\mathbb R}
\def\bE{\mathbb E}

\def\bN{\mathbb N}

\def\Ccal{\mathcal C}
\def\Dcal{\mathcal D}

\def\Kcal{\mathcal K}

\def\Mcal{\mathcal M}

\def\Ocal{\mathcal O}

\def\Scal{\mathcal S}

\def\Xcal{\mathcal X}
\def\Tscr{\mathscr T}
\def\Rscr{\mathscr R}
\def\Pscr{\mathscr P}

\def\dim{{\rm dim}}
\def\rank{{\rm rank}}
\def\diag{{\rm diag}}

\def\tr{{\rm trace}}
\def\dim{{\rm dim}}

\def\argmin{\mathop{\rm argmin}}
\def\one{\mathbf 1}
\def\zero{\mathbf 0}

\def\hat{\widehat}

\def\blist#1#2#3
{    \begin{list}{}{
     \setlength{\parsep}{0pt}
     \setlength{\leftmargin}{#1 pt}
     \setlength{\listparindent}{0pt}
     \setlength{\itemindent}{\listparindent}
     \setlength{\labelsep}{#3 pt}
     \setlength{\labelwidth}{\leftmargin}
     \addtolength{\labelwidth}{-\labelsep}
     \addtolength{\labelwidth}{\itemindent}
     \setlength{\rightmargin}{#2 pt}   }   }

\def\elist{\end{list}}

\def\beq{\begin{equation}}
\def\eeq{\end{equation}}

\def\bes{\begin{eqnarray*}}
\def\ees{\end{eqnarray*}}

\begin{document}

\title{Distance Shrinkage and Euclidean Embedding via Regularized Kernel Estimation}

\date{(\today)}

\author{\\
Luwan Zhang$^\ast$, Grace Wahba$^\dag$ and Ming Yuan$^{\ast,\ddag}$\\
\\
Department of Statistics\\
University of Wisconsin-Madison}

\footnotetext[1]{
Research supported in part by NSF Career Award DMS-1321692 and FRG Grant DMS-1265202}
\footnotetext[2]{
Research supported in part by NIH Grant EY09946 and NSF Grant DMS1308877.}
\footnotetext[3]{
Address for correspondence: Department of Statistics, University of Wisconsin-Madison, 1300 University Avenue, Madison, WI 53706.}

\maketitle
\newpage

\begin{abstract}
Although recovering an Euclidean distance matrix from noisy observations is a common problem in practice, how well this could be done remains largely unknown. To fill in this void, we study a simple distance matrix estimate based upon the so-called regularized kernel estimate. We show that such an estimate can be characterized as simply applying a constant amount of shrinkage to all observed pairwise distances. This fact allows us to establish risk bounds for the estimate implying that the true distances can be estimated consistently in an average sense as the number of objects increases. In addition, such a characterization suggests an efficient algorithm to compute the distance matrix estimator, as an alternative to the usual second order cone programming known not to scale well for large problems. Numerical experiments and an application in visualizing the diversity of Vpu protein sequences from a recent HIV-1 study further demonstrate the practical merits of the proposed method.
\end{abstract}

\section{Introduction}
The problem of recovering an Euclidean distance matrix from noisy or imperfect observations of pairwise dissimilarity scores between a set of objects arises naturally in many different contexts. It allows us to map objects from an arbitrary domain to Euclidean spaces, and therefore makes them amenable for subsequent statistical analyses, and also provides tools for visualization. Consider, for example, evaluating (dis)similarity between molecular sequences. A standard approach is through sequence alignment and measuring the (dis)similarity between a pair of sequences using their corresponding alignment score (see, Durbin et al., 1998). Although encoding invaluable insights into the relationship between sequences, it is well known that these scores do not correspond directly to a distance metric in the respective sequence space and therefore cannot be employed in kernel based learning methods. Similarly, there are also numerous other instances where it is possible to derive similarity or dissimilarity scores for pairs of objects from expert knowledge or other information, which, if successfully converted into positive semi-definite kernels or Euclidean distances, could allow themselves to play an important role in a myriads of statistical and computational analyses (e.g., Sch\"olkopf and Smola, 2002; Sz\'ekely, Rizzo and Bakirov, 2007).

A canonical example where this type of problem occurs is multidimensional scaling which aims to place each object in a low dimensional Euclidean space such that the between-object distances are preserved as well as possible. As such it also forms the basis for several other more recent approaches to nonlinear dimension reduction and manifold learning. See, Sch\"olkopf (1998), Tenanbaum, De Silva and Langford (2000), Lu et al. (2005), Venna and Kaski (2006), Weinberger et al. (2007), Chen and Buja (2009, 2013) among others. Despite the popularity of multidimensional scaling, very little is known about to what extent the distances among the embedded points could faithfully reflect the true pairwise distances when observed with noises; and it is largely used only as an exploratory data analysis tool.

Another example where it is of interest to reconstruct an Euclidean distance matrix is the determination of molecular structures using nuclear magnetic resonance (NMR, for short) spectroscopy, a technique pioneered by Nobel laureate Kurt W\"uthrich (see, e.g., W\"uthrich, 1986). As demonstrated by W\"uthrich, distances between atoms could be inferred from chemical shifts measured by NMR spectroscopy. These distances obviously need to conform to a three dimensional Euclidean space yet experimental data on distances are inevitably noisy and as a result, the observed distances may not translate directly into locations of these atoms in a stable structure. Therefore, this becomes a problem of recovering an Euclidean distance matrix in 3D from noisy observations of pairwise distances. Similar problems also occur in graph realization and Euclidean representation of graphs where the goal is to embed the vertex set of a graph in an Euclidean space in such a fashion that the distance between two embedded vertexes matches their corresponding edge weight (see, e.g., Pouzet, 1979). While an exact embedding of a graph is typically of very high dimension, it is useful in some applications to instead seek approximate yet low dimensional embeddings (see, e.g., Roy, 2010).

More specifically, let $\{O_i: i=1,2,\ldots,n\}$ be a collection of objects from domain $\Ocal$ which could be the coordinates of atoms in the case of molecular structure determination using NMR spectroscopy, or the vertex set of a graph in the case of graph realization. Let $\{x_{ij}: 1\le i<j\le n\}$ be the observed dissimilarity scores between them such that 
$$
x_{ij}= d_{ij}+\varepsilon_{ij},\qquad 1\le i<j\le n,
$$
where $\varepsilon_{ij}$s are the measurement errors and $D=(d_{ij})_{1\le i,j\le n}$ is a so-called Euclidean distance matrix in that there exist points $p_1,\ldots, p_n\in \bR^{k}$ for some $k\in \bN$ such that
\beq
\label{eq:embed}
d_{ij}=\|p_i-p_j\|^2,\qquad 1\le i<j\le n;
\eeq
see, e.g., Darrotto (2013). Here $\|\cdot\|$ stands for the usual Euclidean distance . Our goal is to estimate the Euclidean distance matrix $D$ from the observed matrix $X=(x_{ij})_{1\le i,j\le n}$ where we adopt the convention that $x_{ji}=x_{ij}$ and $x_{ii}=0$. In the light of (\ref{eq:embed}), $D$ can be identified with the points $p_i$s, which suggests an embedding of $O_i$s in $\bR^{k}$. Obviously, if $O_i$s can be embedded in the Euclidean space of a particular dimension, then it is also possible to embed them in a higher dimensional Euclidean space. We refer to the smallest $k$ in which such an embedding is possible as the embedding dimension of $D$, denoted by $\dim(D)$. As is clear from the aforementioned examples, oftentimes, either the true Euclidean distance matrix $D$ itself is of low embedding dimension; or we are interested in an approximation of $D$ that allows for a low dimensional embedding. Such is the case, for example, for molecular structure determination where the the embedding dimension of the true distance matrix $D$ is necessarily three. Similarly, for multidimensional scaling or graph realization, we typically are interested in mapping objects in two or three dimensions.

Recall that
$$
d_{ij}=p_i^\top p_i+p_j^\top p_j-2p_i^\top p_j,
$$
which relates $D$ to the so-called kernel (or Gram) matrix $K=(p_i^\top p_j)_{1\le i,j\le n}$. Furthermore, it is also clear that the embedding dimension $\dim(D)$ equals to $\rank(K)$. Motivated by this correspondence between an Euclidean distance matrix and a kernel matrix, we consider estimating $D$ by $\hat{D}=(\hat{d}_{ij})_{1\le i,j\le n}$ where
\begin{equation}
\label{eq:defD}
\widehat{d}_{ij}=\left\langle \hat{K}, (e_i-e_j)(e_i-e_j)^\top\right\rangle=\hat{k}_{ii}+\hat{k}_{jj}-2\hat{k}_{ij}.
\end{equation}
Here $\langle A, B\rangle=\tr(A^\top B)$, $e_i$ is the $i$th column vector of the identity matrix, and $\hat{K}=(\hat{k}_{ij})_{1\le i,j\le n}$ is the the so-called regularized kernel estimate; see, e.g., Lu et al. (2005) and Weinberger et al. (2007). More specifically,
\begin{equation}
\label{eq:rke}
\widehat{K}=\argmin_{M\succeq 0}\left\{\sum_{1\le i<j\le n}\left(x_{ij}- \left\langle M, (e_i-e_j)(e_i-e_j)^\top\right\rangle\right)^2+\lambda_n \tr(M)\right\},
\end{equation}
where $\lambda_n\ge 0$ is a tuning parameter that balances the tradeoff between goodness-of-fit and the preference towards an estimate with smaller trace norm. Hereafter, we write $M\succeq 0$ to indicate that a matrix $M$ is positive semi-definite. The trace norm penalty used in defining $\hat{K}$ encourages low-rankness of the estimated kernel matrix and hence low embedding dimension of $\hat{D}$. See, e.g., Lu et al. (2005), Yuan et al. (2007), Negahban and Wainwright (2011), Rohde and Tsybakov (2011), and Lu, Monteiro and Yuan (2012) among many others for similar use of this type of penalty. The goal of the current article is to study the operating characteristics and statistical performance of the estimate $\hat{D}$ defined by (\ref{eq:defD}).

A fundamental difficulty in understanding the behavior of the proposed  distance matrix estimate $\hat{D}$ comes from the simple observation that a kernel is not identifiable given pairwise distances alone, even without noise, as the latter is preserved under translation while the former is not. Therefore, it is not clear what exactly $\hat{K}$ is estimating, and subsequently what the relationship between $\hat{D}$ and $D$ is. To address this challenge, we introduce a notion of minimum trace kernel to resolve the ambiguity associated with kernel estimation. Understanding of this concept allows us to more directly and explicitly characterize $\hat{D}$ as first applying a constant amount of shrinkage to all observed distances; and then projecting the shrunken distances to an Euclidean distance matrix. Because the distance between a pair of points shrinks when they are projected onto a linear subspace, this characterization offers a geometrical explanation to the ability of $\hat{D}$ to induce low dimensional embeddings. In addition, this direct characterization of $\hat{D}$ also suggests an efficient way to compute it using a version of Dykstra's alternating projection algorithm thanks to the special geometric structure of $\Dcal_n$, the set of $n\times n$ distance matrices. See, e.g., Glunt et al. (1990). Obviation of semidefinite programming, and more generally second order cone programmingÕs computational expense is the principal advantage of this alternating projection technique. Furthermore, based on this explicit characterization, we establish statistical risk bounds for the discrepancy $\hat{D}-D$ and show that the true distances can be recovered consistently in average if $D$ allows for (approximate) low dimensional embeddings. 

The rest of the paper is organized as follows. In Section \ref{sec:meth}, we discuss in details the shrinkage effect of the estimate $\hat{D}$ by exploiting the duality between a kernel matrix and an Euclidean distance matrix. Taking advantage of our explicit characterization of $\hat{D}$ and the geometry of the convex cone of Euclidean distance matrices, Section \ref{sec:th} establishes risk bounds for $\hat{D}$ and Section \ref{sec:alg} describes how $\hat{D}$ can be computed using an efficient alternating projection algorithm. The merits of $\hat{D}$ is further illustrated via numerical examples, both simulated and real, in Section \ref{sec:num}. All proofs are relegated to Section \ref{sec:proof}.

\section{Distance Shrinkage}
\label{sec:meth}

In this section, we show that there is a one-to-one correspondence between an Euclidean distance matrix and a so-called minimum trace kernel; and exploit this duality explicitly to characterize $\hat{D}$.

\subsection{Minimum Trace Kernels}
\label{sec:mintr}
Despite the popularity of regularized kernel estimate $\hat{K}$, rather little is known about its statistical performance. This is perhaps in a certain sense inevitable because a kernel is not identifiable given pairwise distances alone. To resolve this ambiguity, we introduce the concept of minimum trace kernel, and show that $\hat{K}$ is targeting at the unique minimum trace kernel associated with the true Euclidean distance matrix.

Recall that any $n\times n$ positive semidefinite matrix $K$ can be identified with a set of points $p_1,\ldots, p_n\in \bR^{k}$ for some $k\in \bN$ such that $K=PP^{\top}$ where $P=(p_1,\ldots, p_n)^\top$. At the same time, these points can also be associated with an $n\times n$ Euclidean distance matrix $D=(d_{ij})_{1\le i,j\le n}$ where
$$
d_{ij}=\|p_i-p_j\|^2,\qquad 1\le i<j\le n.
$$
Obviously,
$$
d_{ij}=\langle K, B_{ij}\rangle,
$$
where
$$
B_{ij}=(e_i-e_j)(e_i-e_j)^\top.
$$

It is clear that any positive semi-definite matrix $M$ can be a kernel matrix and therefore translated uniquely into a distance matrix. In other words,
$$
\Tscr(M)=\diag(M)\one^\top +\one\diag(M)^\top-2M=(m_{ii}+m_{jj}-2m_{ij})_{1\le i,j\le n}
$$
is a surjective map from the set $\Scal_n$ of $n\times n$ positive semi-definite matrices to $\Dcal_n$. Hereafter, we write $\one$ as a vector of ones of conformable dimension. The map $\Tscr$, however, is not injective because, geometrically, translation of the embedding points results in different kernel matrix yet the distance matrix remains unchanged. As a result, it may not be meaningful, in general, to consider reconstruction of a kernel matrix from dissimilarity scores alone.

It turns out that one can easily avoid such an ambiguity by requiring the embeddings to be centered in that $P^\top \one =\zero$ where $\zero$ is a vector of zeros of conformable dimension. We note that even with the centering, the embeddings as represented by $P$ for any given Euclidean distance matrix still may not be unique as distances are invariant to rigid motions. However, their corresponding kernel matrix, as the following result shows, is indeed uniquely defined. Moreover the kernel matrix can be characterized as having the smallest trace among all kernels that correspond to the same distance matrix, hence will be referred to as the minimum trace kernel.

\begin{theorem}
\label{th:mintr}
Let $D$ be an $n\times n$ distance matrix. Then the preimage of of $D$ under $\Tscr$
$$
\Mcal(D)=\{M\in \Scal_n: \Tscr(M)=D\}
$$
is convex; and $-JDJ/2$ is the unique solution to following convex program
$$
\argmin_{M\in \Mcal(D)} \tr(M),
$$
where $J=I-(\one\one^\top/n)$. In addition, if $p_1,\ldots, p_n\in \bR^n$ is an embedding of $D$ such that $p_1+\ldots+p_n=\zero$, then $PP^\top=-JDJ/2$, where $P=(p_1,\ldots, p_n)^\top$.
\end{theorem}

\vskip 10pt

In the light of Theorem \ref{th:mintr}, $\Tscr$ is bijective when restricted to the set of minimum trace kernels:
$$
\Kcal=\{M\succeq 0: \tr(M)\le \tr(A),\quad \forall A\in \Mcal(\Tscr(M))\}.
$$
and its inverse is $\Rscr(M)=-JMJ/2$ as a map from distance matrices to kernels with minimum trace. From this viewpoint, the regularized kernel estimate $\hat{K}$ intends to estimate $\Rscr(D)$ instead of the original data-generating kernel. In addition, it is clear that
\begin{proposition}
\label{pr:rke}
For any $\lambda_n>0$, the regularized kernel estimate $\widehat{K}$ as defined in (\ref{eq:rke}) is a minimum trace kernel. In addition, any embedding $\hat{P}$ of $\hat{K}$, that is $\hat{K}=\hat{P}\hat{P}^\top$, is necessarily centered so that $\hat{P}^\top \one=\zero$.
\end{proposition}

The relationships among the data-generating kernel $K$, $D$, $\Rscr(D)$, regularized kernel estimate $\hat{K}$ as defined by (\ref{eq:rke}), and the distance matrix estimate $\hat{D}$ as defined by (\ref{eq:defD}) can be described by Figure \ref{fig:relation}.

\begin{figure}[htbp]
\begin{center}
\begin{tikzpicture}
  \matrix (m) [matrix of math nodes, row sep=3em, column sep=3em]
    { K & D  & \Rscr(D)   \\
          & \hat{D} & \hat{K} \\ };
      \path[-stealth]
     (m-1-1) edge[right] node[above] {$\Tscr$} (m-1-2) 
     (m-1-2) edge[right]   (m-1-3)
     (m-1-3) edge[left]   (m-1-2)
     (m-2-3) edge[dashed,->]  (m-1-3)
     (m-2-2) edge[right] (m-2-3)
                  edge[dashed,->]  (m-1-2)
      (m-2-3) edge[left] (m-2-2);              
\end{tikzpicture}
\end{center}
\caption{Relationships among $K$, $D$, $\Rscr(D)$, $\hat{K}$ and $\hat{D}$: the true distance matrix $D$ is determined by the data-generating kernel $K$; there is a one-to-one correspondence between $D$ and the minimum trace kernel $\Rscr(D)$. Similarly, there is a one-to-one correspondence between $\hat{D}$ and $\hat{K}$ which are estimate of $D$ and $\Rscr(D)$ respectively.}
\label{fig:relation}
\end{figure}
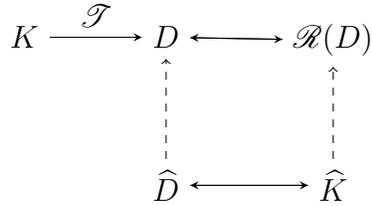

\subsection{Distance Shrinkage}
\label{sec:dist}

We now study the properties of the proposed distance matrix estimate given by (\ref{eq:defD}). Following Theorem \ref{th:mintr}, $\widehat{D}$ can be equivalently expressed as
\begin{equation}
\label{eq:rkels}
\widehat{D}=\argmin_{M\in \Dcal_n}\left\{{1\over 2}\|X-M\|^2_{\rm F}+\lambda_n \tr\left(-{1\over 2}JMJ\right)\right\},
\end{equation}
where $\|\cdot\|_{\rm F}$ stands for the usual matrix Frobenius norm. It turns out that $\hat{D}$ actually allows for a more explicit expression.

To this end, observe that the set $\Dcal_n$ of $n\times n$ Euclidean distance matrices is a closed convex cone (Sch\"onberg, 1935; Young and Householder, 1938). Let $\Pscr_{\Dcal_n}$ denote the projection to $\Dcal_n$ in that
$$
\Pscr_{\Dcal_n}(A)=\argmin_{M\in \Dcal_n}\|A-M\|_{\rm F}^2.
$$
for $A\in \bR^{n\times n}$. Then
\begin{theorem}
\label{th:equiv}
Let $\widehat{D}$ be defined by (\ref{eq:defD}). Then
$$
\widehat{D}=\Pscr_{\Dcal_n}\left(X-{\lambda_n\over 2n} D_0\right)
$$
where $D_0$ is an Euclidean distance matrix whose diagonal elements are zero and off-diagonal entries are ones.
\end{theorem}

\vskip 10pt

Theorem \ref{th:equiv} characterizes $\hat{D}$ as the projection of $X-(\lambda_n/2n) D_0$ to an Euclidean distance matrix. Therefore, it can be computed as soon as we can evaluate the projection onto the closed convex set $\Dcal_n$. As shown in Section \ref{sec:alg}, this could be done efficiently using an alternating projection algorithm thanks to the geometric structure of $\Dcal_n$. In addition, subtraction of $(\lambda_n/2n) D_0$ from $X$ amounts to applying a constant shrinkage to all observed pairwise distances. Geometrically, distance shrinkage can be the result of projecting points in an Euclidean space onto a lower dimensional linear subspace, and therefore encourages low dimensional embeddings. We now look at the specific example when $n=3$ to further illustrate such an effect.

In the special case of $n=3$ points, the projection to Euclidean distance matrices can be computed analytically. Let
$$
X=\left[\begin{array}{ccc}0& x_{12}& x_{13}\\ x_{12}& 0& x_{23}\\ x_{13}& x_{23}& 0\end{array}\right]
$$
be the observed distance matrix. We now determine the embedding dimension of $\Pscr_{\Dcal_3}(X-\eta D_0)$.

Let
$$
Q={1\over 3+\sqrt{3}}\left[\begin{array}{ccc}2+\sqrt{3}& -1& -(1+\sqrt{3})\\ -1& 2+\sqrt{3}& -(1+\sqrt{3})\\ -(1+\sqrt{3})& -(1+\sqrt{3})& -(1+\sqrt{3})\end{array}\right]
$$
be a $3\times 3$ Householder matrix. Then, for a $3\times 3$ symmetric hollow matrix $X$,
$$
QXQ=\left[\begin{array}{ccc}-{1\over 3}x_{12}-{1+\sqrt{3}\over 3}x_{13}+{1+\sqrt{3}\over 6+3\sqrt{3}}x_{23}& {2\over 3}x_{12}-{1\over 3}x_{13}-{1\over 3}x_{23}& \ast\\ {2\over 3}x_{12}-{1\over 3}x_{13}-{1\over 3}x_{23}& -{1\over 3}x_{12}+{1+\sqrt{3}\over 6+3\sqrt{3}}x_{13}-{1+\sqrt{3}\over 3}x_{23}& \ast\\ \ast& \ast& \ast\end{array}\right],
$$
where we only give the $2\times 2$ leading principle matrix of $QXQ$ and leave the other entries unspecified. As shown by Hayden and Wells (1988), the minimal embedding dimension of $\Pscr_{\Dcal_3}(X)$ can be determined by the eigenvalues of the principle matrix.

More specifically, denote by
$$
\tilde{D}(X)=\left[\begin{array}{cc}{1\over 3}x_{12}+{1+\sqrt{3}\over 3}x_{13}-{1+\sqrt{3}\over 6+3\sqrt{3}}x_{23}& -{2\over 3}x_{12}+{1\over 3}x_{13}+{1\over 3}x_{23}\\ -{2\over 3}x_{12}+{1\over 3}x_{13}+{1\over 3}x_{23}& {1\over 3}x_{12}-{1+\sqrt{3}\over 6+3\sqrt{3}}x_{13}+{1+\sqrt{3}\over 3}x_{23}\end{array}\right],
$$
and
$$
\tilde{D}(X)=U\left[\begin{array}{cc}\alpha_1& 0\\ 0& \alpha_2\end{array}\right]U^\top
$$
its eigenvalue decomposition. Write 
\begin{equation}
\label{eq:defDelx}
\Delta_x:=\sqrt{2[(x_{12} - x_{13})^2+(x_{12} - x_{23})^2+(x_{13} - x_{23})^2]}.
\end{equation}
Then, it can be calculated that
\begin{equation}
\label{eq:eigen}
\alpha_1 = {(x_{12}+x_{13} + x_{23})+\Delta_x \over 3}, \qquad {\rm and}\qquad \alpha_2 = {(x_{12}+x_{13} + x_{23})-\Delta_x \over 3}.
\end{equation}
In the light of Theorem 6.1 of Glunt et al. (1990), we have

\begin{proposition}
\label{pr:proj3}
$$
\dim(\Pscr_{\Dcal_3}(X))=\left\{\begin{array}{ll}2& {\rm if\ }x_{12}+x_{13} + x_{23}>\Delta_x\\ 1& {\rm if\ }-{1\over 2}\Delta_x<x_{12}+x_{13} + x_{23}\le \Delta_x\\ 0 & {\rm otherwise}\end{array}\right.,
$$
where $\Delta_x$ is given by (\ref{eq:defDelx}), and $\dim(\Pscr_{\Dcal_3}(X))=0$ means $\Pscr_{\Dcal_3}(X)=\zero$.
\end{proposition}

\vskip 10pt

To appreciate the effect of distance shrinkage, consider the case when $\Pscr_{\Dcal_3}(X)$ has a minimum embedding dimension of two. By Proposition \ref{pr:proj3}, this is equivalent to assuming $\alpha_2>0$. Observe that
$$
\tilde{D}(X-\eta D_0)=\tilde{D}(X)-\eta I_2.
$$
The eigenvalues of $\tilde{D}(X-\eta D_0)$ are therefore $\alpha_1-\eta$ and $\alpha_2-\eta$ where $\alpha_1\ge \alpha_2$ are the eigenvalues of $\tilde{D}(X)$ as given by (\ref{eq:eigen}). This indicates that, by applying sufficient amount of distance shrinkage, we can reduce the minimum embedding dimension as illustrated in Figure \ref{fig:shrink}.

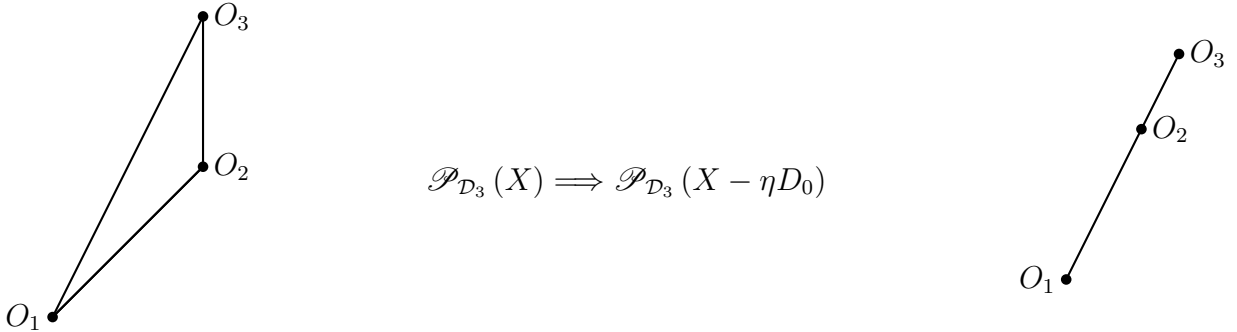
\begin{figure}[htbp]
\begin{minipage}{0.2\textwidth}
\begin{tikzpicture}
\draw[thick,-] (0,0) -- (2,4);
\draw[thick,-] (0,0) -- (2,2);
\draw[thick,-] (2,4) -- (2,2);
\draw[thick,-] (0,0) -- (2,2); 
\fill (0,0) circle (2pt) node[left] {$O_1$};
\fill (2,2) circle (2pt) node[right] {$O_2$};
\fill (2,4) circle (2pt) node[right] {$O_3$};
\end{tikzpicture}
\end{minipage}
\begin{minipage}{0.6\textwidth}
$$
\Pscr_{\Dcal_3}\left(X\right)\Longrightarrow\Pscr_{\Dcal_3}\left(X-\eta D_0\right)
$$
\end{minipage}
\begin{minipage}{0.17\textwidth}
\begin{tikzpicture}
\draw[thick,-] (0,0) -- (1,2);
\draw[thick,-] (1,2) -- (1.5,3);
\fill (0,0) circle (2pt) node[left] {$O_1$};
\fill (1,2) circle (2pt) node[right] {$O_2$};
\fill (1.5,3) circle (2pt) node[right] {$O_3$};
\end{tikzpicture}
\end{minipage}
\caption{Effect of distance shrinkage when $n=3$.}
\label{fig:shrink}
\end{figure}

More specifically,
\begin{itemize}
\item If
$$
{1\over 3}(x_{12}+x_{13}+x_{23})-{\Delta_x\over 3}\le \eta<{1\over 3}(x_{12}+x_{13}+x_{23})+{2\Delta_x\over 3},
$$
then the minimum embedding dimension of $\Pscr_{\Dcal_3}(X-\eta D_0)$ is one.
\item If
$$
\eta\ge {1\over 3}(x_{12}+x_{13}+x_{23})+{2\Delta_x\over 3},
$$
then the minimum embedding dimension of $\Pscr_{\Dcal_3}(X-\eta D_0)$ is zero;
\end{itemize}

\section{Estimation Risk}
\label{sec:th}
The previous section provides an explicit characterization of the proposed distance matrix estimate $\hat{D}$ as a distance shrinkage estimator. We now take advantage this characterization to establish statistical risk bounds for $\widehat{D}$.

\subsection{Estimation Error for Distance Matrix}
A natural measure of the quality of a distance matrix estimate $\tilde{D}$ is the averaged squared error of all pairwise distances:
$$
L(\tilde{D},D):={2\over n(n-1)}\sum_{1\le i<j\le n} \left(\tilde{d}_{ij}-d_{ij}\right)^2.
$$
It is clear that when both $\tilde{D}$ and $D$ are $n\times n$ Euclidean distance matrices,
$$
L(\tilde{D},D)={1\over n(n-1)}\|\tilde{D}-D\|_{\rm F}^2.
$$
For convenience, we shall now consider bounding $\|\hat{D}-D\|_{\rm F}^2$. Taking advantage of the characterization of $\hat{D}$ as a projection onto the set of $n\times n$ Euclidean distance matrices, we can derive the following oracle inequality.

\begin{theorem}
\label{th:oracle}
Let $\widehat{D}$ be defined by (\ref{eq:defD}). Then for any $\lambda_n$ such that $\lambda_n\ge 2\|X-D\|$,
$$
\|\widehat{D}-D\|_{\rm F}^2\le \inf_{M\in \Dcal_n}\left\{\|M-D\|_{\rm F}^2+{9\over 4}\lambda_n^2(\dim(M)+1)\right\},
$$
where $\|\cdot\|$ stands for the matrix spectral norm.
\end{theorem}

\vskip 10pt

Theorem \ref{th:oracle} gives a deterministic upper bound for the error of $\hat{D}$, $\|\hat{D}-D\|_{\rm F}^2$ in comparison with that of an arbitrary approximation to $D$. More specifically, let $\tilde{D}$ be the closest Euclidean distance matrix with embedding dimension $r$ to $D$, in terms of Frobenius norm. Then Theorem \ref{th:oracle} implies that with sufficiently large tuning parameter $\lambda_n$,
$$
L(\hat{D},D)\le L(\tilde{D}, D)+{Cr\lambda_n^2\over n^2},
$$
for some constant $C>0$. In particular, if $D$ itself is embedding dimension $r$, then
$$
L(\hat{D},D)\le {Cr\lambda_n^2\over n^2}.
$$

More explicit bounds for the estimation error can be derived from this general result. Consider, for example, the case when the observed pairwise distances are the true distances subject to additive noise:
\begin{equation}
\label{eq:add}
x_{ij}=d_{ij}+\varepsilon_{ij},\qquad 1\le i<j\le n,
\end{equation}
where the measurement errors $\varepsilon_{ij}$s are independent with mean $\bE(\varepsilon_{ij})=0$ and variance ${\rm var}(\varepsilon_{ij})=\sigma^2$. Assume that the distributions of measurement errors have light tails such that
\begin{equation}
\label{eq:subgauss}
\bE(\varepsilon_{ij})^{2m}\le (c_0m)^m, \qquad \forall m\in \bN
\end{equation}
for some constant $c_0>0$. Then the spectral norm of $X-D$ satisfies
$$
\|X-D\|=2\sigma\left(\sqrt{n}+O_p(n^{-1/6})\right).
$$
See, e.g., Sinai and Soshnikov (1998). Thus, 
\begin{corollary}
\label{co:add}
Let $\widehat{D}$ be defined by (\ref{eq:defD}). Under the model given by (\ref{eq:add}) and (\ref{eq:subgauss}), if $\lambda_n=4\sigma(n^{1/2}+1)$, then with probability tending to one,
$$
\|\widehat{D}-D\|_{\rm F}^2\le \inf_{M\in \Dcal_n}\left\{\|M-D\|_{\rm F}^2+{36n\sigma^2}(\dim(M)+1)\right\},
$$
as $n\to\infty$. In particular, if $\dim(D)=r$, then with probability tending to one,
$$
\|\widehat{D}-D\|_{\rm F}^2\le 36n\sigma^2(r+1).
$$
\end{corollary}

\vskip 10pt

In other words, under the model given by (\ref{eq:add}) and (\ref{eq:subgauss}),
$$
L(\hat{D},D)\le L(\tilde{D}, D)+{Cr\sigma^2\over n},
$$
for some constant $C>0$, where as before, $\tilde{D}$ is the closest Euclidean distance matrix to $D$ with embedding dimension $r$. In particular, if $D$ itself is embedding dimension $r$, then
$$
L(\hat{D},D)\le {Cr\sigma^2\over n}.
$$

\subsection{Low Dimensional Approximation}

As mentioned before, in some applications, the chief goal may not be to recover $D$ itself but rather its embedding in a prescribed dimension. This is true, in particular, for multidimensional scaling and graph realization where we are often interested in embedding a distance matrix in $\bR^2$ or $\bR^3$. Following the classical multidimensional scaling, a parameter of interest in these cases is
$$
D_r:=\argmin_{M\in \Dcal_n(r)} \|J(D-M)J\|_{\rm F}^2,
$$
where $\Dcal_n(r)$ is the set of all $n\times n$ Euclidean distance matrices of embedding dimension at most $r$. An obvious estimate of $D_r$ can be derived by replacing $D$ with $\hat{D}$:
\begin{equation}
\label{eq:defDr}
\hat{D}_r:=\argmin_{M\in \Dcal_n(r)} \|J(\hat{D}-M)J\|_{\rm F}^2.
\end{equation}
Similar to the classical multidimensional scaling, the estimate $\hat{D}_r$ can be computed more explicitly as follows. Let $\hat{K}$ be the regularized kernel estimate corresponding to $\hat{D}$, and $\hat{K}=U\Gamma U^\top$ be its eigenvalue decomposition with $\Gamma=\diag(\gamma_1,\gamma_2,\ldots)$ and $\gamma_1\ge \gamma_2\ge\ldots$. Then $\hat{D}_r=\Tscr(\hat{K}_r)$ where $\hat{K}_r=U\diag(\gamma_1,\ldots,\gamma_r,0,\ldots) U^\top$.
 
The risk bounds we derived for $\hat{D}$ can also be translated into that for $\hat{D}_r$. More specifically,

\begin{corollary}
\label{co:embed}
Let $\hat{D}_r$ be defined by (\ref{eq:defDr}) where $\hat{D}$ is given by (\ref{eq:defD}) with $\lambda_n\ge 2\|X-D\|$. Then there exists a numerical constant $C>0$ such that
$$
\|J(\hat{D}_r-D)J\|_{\rm F}^2\le C\left(\min_{M\in \Dcal_n(r)} \|J(D-M)J\|_{\rm F}^2+{\lambda_n^2r}\right),
$$
In particular, under the model given by (\ref{eq:add}) and (\ref{eq:subgauss}), if $\lambda_n=4\sigma(n^{1/2}+1)$, then with probability tending to one,
$$
\|J(\hat{D}_r-D)J\|_{\rm F}^2\le C\left(\min_{M\in \Dcal_n(r)} \|J(D-M)J\|_{\rm F}^2+{nr\sigma^2}\right).
$$
\end{corollary}

\vskip 10pt

\section{Computation}
\label{sec:alg}

It is not hard to see that the optimization problem involved in defining the regularized kernel estimate can be formulated as a second order cone program (see, e.g., Lu et al. 2005; Yuan et al., 2007). This class of optimization problems can be readily solved using generic solvers such as \verb+SDPT3+ (Toh, Todd and Tutuncu, 1999; Tutuncu, Toh and Todd, 2003). Although in principle, these problems can be solved in polynomial time, on the practical side, the solvers are known not to scale well to large problems.  Instead of starting from the regularized kernel estimate, as shown in Section \ref{sec:th}, $\hat{D}$ can be directly computed as a projection onto the set of Euclidean distance matrices. Taking advantage of this direct characterization and the particular geometric structure of the closed convex cone $\Dcal_n$, we can devise a more efficient algorithm to compute $\hat{D}$.


We shall adopt, in particular, an alternating projection algorithm introduced by Dykstra (1983). Dykstra's algorithm is a refinement of the von Neumann alternating projection algorithm specifically designed to compute projection onto the intersection of two closed convex sets by constructing a sequence of projections to the two sets alternatively.
 
\begin{algorithm}[htbp]
\caption{Dykstra's alternating projection algorithm: $\Pscr_{\Ccal_1}$ and $\Pscr_{\Ccal_2}$ are the projections onto $\Ccal_1$ and $\Ccal_2$ respectively.}
\label{alg:dykstra}
\KwData{$x$.}
\KwResult{Projection of $x$ onto the intersection of two closed convex set $\Ccal_1$ and $\Ccal_2$.}
Initialization: $x_0 = x$, $p_0=0$, $q_0=0$, $k=0$ \;
\Repeat{\emph{a certain convergence criterion is met}} {  
  $s_k \leftarrow \Pscr_{\Ccal_1}(x_{k} + p_{k})$\; 
  $p_{k+1} \leftarrow x_{k} + p_{k}-s_k$ \;
  $x_{k+1} \leftarrow \Pscr_{\Ccal_2}(s_k+q_{k})$ \;  
  $q_{k+1} \leftarrow s_k+q_{k}-x_{k+1}$ \;
  $k\leftarrow k+1$ \;
}
\KwRet{$x_{k+1}$} \;
\end{algorithm}

The idea can also be illustrated by Figure \ref{fig:alt} where the projection of a point onto the intersection of two half-planes is computed.

\begin{figure}[htbp]
\begin{center}
\begin{tikzpicture}
\draw[thick,-] (-1,0) -- (6,0);
\draw[thick,-] (-1,-0.5) -- (6,3);
\draw[dashed,thick,->] (5,0) -- (5,1) node[pos=0.5,right]{$p_1$};
\draw[dashed,thick,->] (4,2) -- (5,0) node[pos=0.5] {$q_1$};
\draw[dashed,thick,->] (4,0) -- (4,3) node[pos=0.5]{$p_2$}; 
\draw[dashed,thick,->] (3.2,1.6) -- (5,-2) node[pos=0.7,left]{$q_2$}; 
\fill (0,0) circle (2pt) node[above] {$\Pscr_{\Ccal_1\cap\Ccal_2}(x_0)$};
\fill (5,1) circle (2pt) node[right] {$x_0$};
\fill (5,0) circle (2pt) node[below] {$s_0$};
\fill (4,2) circle (2pt) node[right] {$x_1$};
\fill (4,3) circle (2pt) node[right] {$x_1+p_1$};
\fill (4,0) circle (2pt) node[below] {$s_1$};
\fill (3.2,1.6) circle (2pt) node[above] {$x_2$};
\fill (5,-2) circle (2pt) node[right] {$s_1+q_1$};
\fill (2.5,0.6) circle (0pt) node[below] {\ldots};
\fill (6,0) circle (0pt) node[right] {$\Ccal_1$};
\fill (6,3) circle (0pt) node[right] {$\Ccal_2$};
\end{tikzpicture}
\end{center}
\caption{Illustration of alternating projection algorithm.}
\label{fig:alt}
\end{figure}
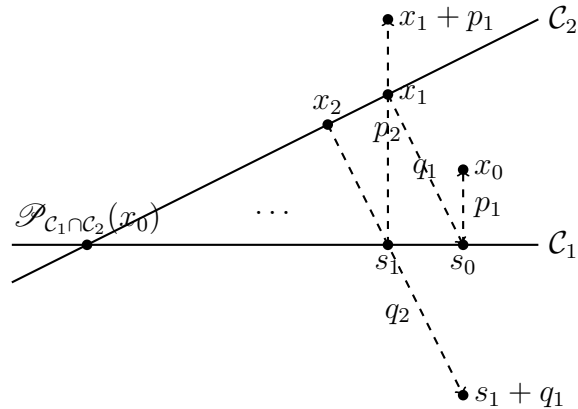

Now consider evaluating $\hat{D}$ which is the projection of $X-\eta_nD_0$ onto $\Dcal_n$. Observe that $\Dcal_n$ is the intersection of two closed convex cones:
$$
\Ccal_1=\{M\in \bR^{n\times n}: JMJ\preceq 0\},
$$
and
$$
\Ccal_2=\{M\in \bR^{n\times n}: \diag(M)=\zero\}.
$$
Dykstra's alternating projection algorithm can then be readily applied with input $X-\eta_nD_0$. The use of alternating projection algorithms is motivated by the fact that although $\Pscr_{\Ccal_1\cap \Ccal_2}$ is difficult to evaluate, projections to $\Ccal_1$ and $\Ccal_2$ actually have explicit form and are easy to compute.

More specifically, for any symmetric matrix $A\in \bR^{n\times n}$, let $\bar{A}_{11}$ be the $(n-1)$th leading principle submatrix of its Householder transform $QAQ$ where
$Q=I-vv^\top/n$ and $v=[1,\ldots,1,1+\sqrt{n}]^\top$. In other words,
$$
A=Q\left[\begin{array}{cc}\bar{A}_{11}& \bar{A}_{12}\\ \bar{A}_{21}& \bar{A}_{22}\end{array}\right]Q
$$
Let $\bar{A}_{11}=U\Gamma U^\top$ be its eigenvalue decomposition. Then
$$
\Pscr_{\Ccal_1}(A)=Q\left[\begin{array}{cc}U\Gamma^+ U^\top& \bar{A}_{12}\\ \bar{A}_{21}& \bar{A}_{22}\end{array}\right]Q
$$
where $\Gamma^+=\diag(\max\{\gamma_{ii},0\})$. See Hayden and Wells (1988). On the other hand, it is clear that $\Pscr_{\Ccal_2}(A)$ simply replaces all diagonal entries of $A$ with zeros.

%
%
%

\section{Numerical Examples}
\label{sec:num}

To illustrate the practical merits of the proposed methods and the efficacy of the algorithm, we conducted several numerical experiments.

\subsection{Sequence Variation of Vpu Protein Sequences}
\label{sec:vpu}

The current work was motivated in part by a recent study on the variation of Vpu (HIV-1 virus protein U) protein sequences and their relationship to preservation of tetherin and CD4 counter-activities (Pickering et al., 2014). Viruses are known for their fast mutation and therefore an important task is to understand the diversity within a viral population. Of particular interest in this study is a Vpu sequence repertoire derived from actively replicating plasma virus from 14 HIV-1-infected individuals. Following standard MACS criteria, five of these individuals can be classified as Long-term nonprogressors, five as rapid progressors, and four as normal progressors, according to how long the progression from seroconversion to AIDS takes. A total of 304 unique amino acid sequences were obtained from this study.

We first performed pairwise alignment between these amino acid sequences using various BLOSUM substitution matrices. The results using different substitution matrices are fairly similar; and to fix ideas, we shall report here analysis based on the BLOSUM62 matrix. These pairwise similarity scores $\{s_{ij}: 1\le i\le j\le n\}$  are converted into dissimilarity scores: 
$$
x_{ij}=s_{ii}+s_{jj}-2s_{ij},\qquad \forall 1\le i<j\le n.
$$
As mentioned earlier, $X=(x_{ij})_{1\le i,j\le n}$ is not an Euclidean distance matrix. To this end, we first applied the classical multidimensional scaling to $X$. The three dimensional embedding is given in the top left panel of Figure \ref{fig:mds1}. The amino acid sequences derived from the same individuals are represented by the same symbol and color. Different colors correspond to the three different classes of disease progression: long-term nonprogressors are represented in red, normal in green, and rapid progressors in purple. For comparison, we also computed $\hat{D}$ with various choices of the tuning parameters. Similar to the observations made by Lu et al. (2005), the corresponding embeddings are qualitatively similar for a wide range of choices of $\lambda_n$. A typical one is given in the top right panel of Figure \ref{fig:mds1}. It is clear that both embeddings share a lot of similarities. For example, sequences derived from the same individual are more similar as they tend to cluster together. The key difference, however, is that the embedding corresponding to $\hat{D}$ suggests an outlying sequence. We went back to the original pairwise dissimilarity scores and identified the sequence as derived from a rapid progressor. It is fairly clear from the original scores that this sequence is different from the others. The minimum dissimilarity score from the particular sequence to any other sequence is 245 whereas the largest score between any other pair of sequences is 215. The histogram of the scores between the sequence and other sequences, or among other sequences are given in the bottom panel of Figure \ref{fig:mds1}.

\begin{figure}[htbp]
\begin{center}
\includegraphics[scale=0.66]{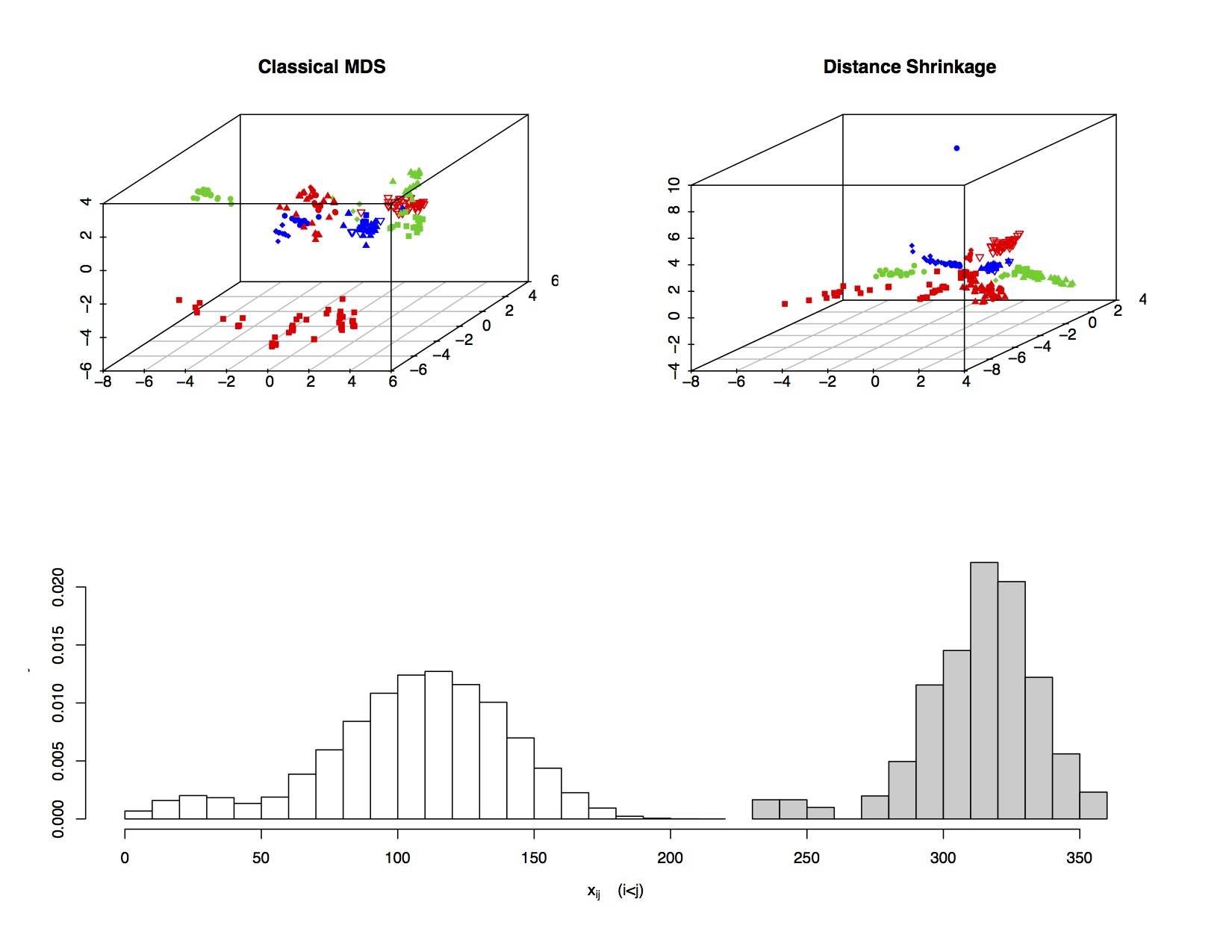}
\caption{Three dimensional embedding for 304 amino acid sequences: the top panels are embeddings from classical multidimensional scaling and distance shrinkage respectively. The histogram of the pairwise dissimilarity scores is given in the bottom panel. The shaded histogram corresponds to those scores between the outlying sequence and the other sequences.}
\label{fig:mds1}
\end{center}
\end{figure}

Given these observations, we now consider the analysis with the outlying sequence removed. To gain insight, we consider different choices of $\lambda_n$ to visually inspect the Euclidean embeddings given by the proposed distance shrinkage. The embeddings given in Figure \ref{fig:mds2} correspond to $\lambda_n$ equals $4000$, $8000$, $12000$ and $16000$ respectively. These embedding are qualitatively similar. 
 
\begin{figure}[htbp]
\begin{center}
\includegraphics[scale=0.66]{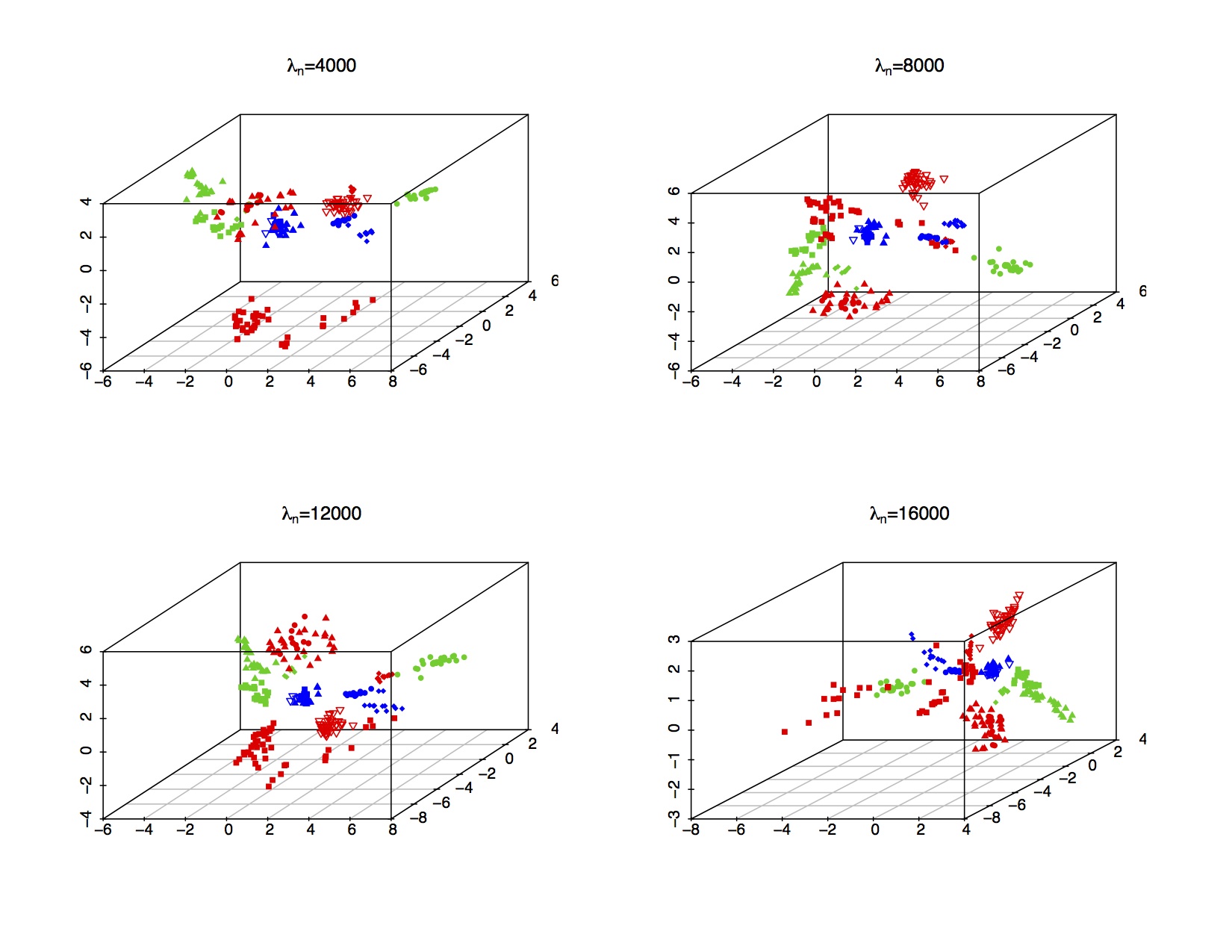}
\caption{Euclidean embedding of 303 amino acid sequences via distance shrinkage: the outlying sequence was removed from the original data and each panel corresponds to different choice of $\lambda_n$.}
\label{fig:mds2}
\end{center}
\end{figure}

\subsection{Simulated Examples}

To further compare the proposed distance shrinkage approach with the classical multidimensional scaling, we carried out several sets of simulation studies. For illustration purposes, we took the setup of the molecular conformation problem discussed earlier. In particular, we considered the problem of protein folding, a process of a random coil conformed to a physically stable three-dimensional structure equipped with some unique characteristics and functions.

We started by extracting the existing data on the 3D structure of the channel-forming trans-membrane domain of Vpu protein from HIV-1 mentioned before. The data obtained from protein data bank (symbol: 1PJE) contains the 3D coordinates of a total of $n=91$ atoms. The exact Euclidean distance matrix $D$ was then calculated from these coordinates. We note that in this case the embedding dimension is known to be three. We generated observations $x_{ij}$ by adding an measurement error $\varepsilon_{ij}\sim N(0,\sigma^2)$ for $1\le i<j\le n$. We considered three different values of $\sigma^2$ -- $0.05$, $0.25$ and $0.5$ respectively, representing relatively high, medium and low signal to noise ratio. For each value of $\sigma^2$, we simulated one hundred datasets and computed for each dataset the Euclidean distance matrix corresponding to the classical multidimensional scaling and the distance shrinkage. We evaluated the performance of each method by the Kruskal's stress defined as $\|\hat{D}-D\|_{\rm F}/ \|D\|_{\rm F}$. The results are summarized by Table \ref{tab:pjesummary}.

\begin{table}[htbp]
\begin{center}
\begin{tabular}{c c c c}
\hline \hline
Signal-to-Noise Ratio & Method & Mean & Standard error \\
\hline
High & Distance Shrinkage & 0.010 & 2.0e-04 \\
\cline{2-4}
 & Classical MDS & 0.078& 9.3e-04 \\ 
\hline
Medium & Distance Shrinkage & 0.024 & 4.8e-04 \\
\cline{2-4}
 & Classical MDS & 0.185 & 2.5e-03 \\
\hline
Low & Distance Shrinkage & 0.035 & 8.4e-04 \\
\cline{2-4}
 & Classical MDS & 0.301 & 3.9e-03 \\
\hline
\end{tabular}
\caption{Kruskal's stress for 1PJE data with measurement error.}
\label{tab:pjesummary}
\end{center}
\end{table}

To better appreciate the difference between the two methods, Figure \ref{fig:1pje} gives the ribbon plot of the protein backbone structure corresponding to the true Euclidean distance matrix and the estimated ones from a typical simulation run with different signal to noise ratios. It is noteworthy that the improvement of the distance shrinkage over the classical multidimensional scaling becomes more evident with higher level of noise.
 
\begin{figure}[htbp]
\begin{center}
\subfigure[Distance Shrinkage, High signal-to-noise ratio]{
\includegraphics[scale=0.25]{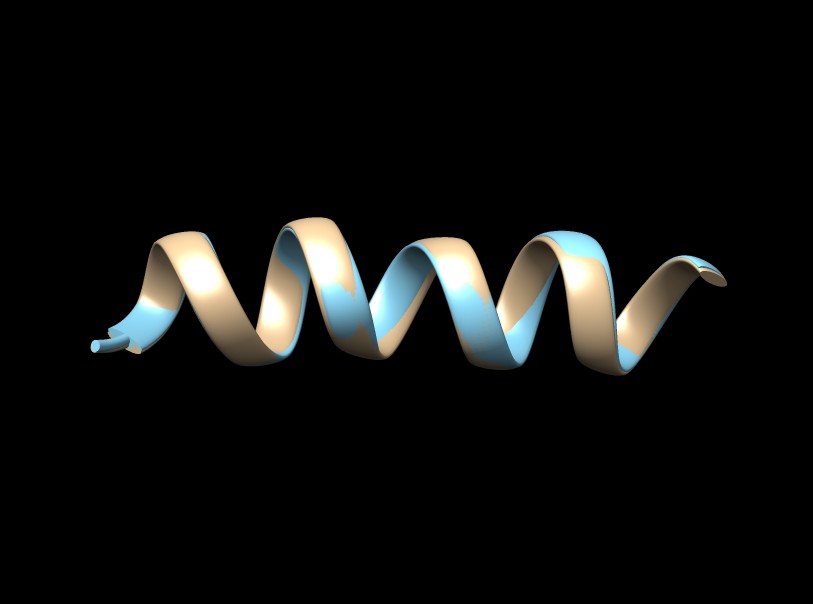}\qquad
\label{subfig:norm2RKE}}
\subfigure[Classical MDS, High signal-to-noise ratio]{
\includegraphics[scale=0.25]{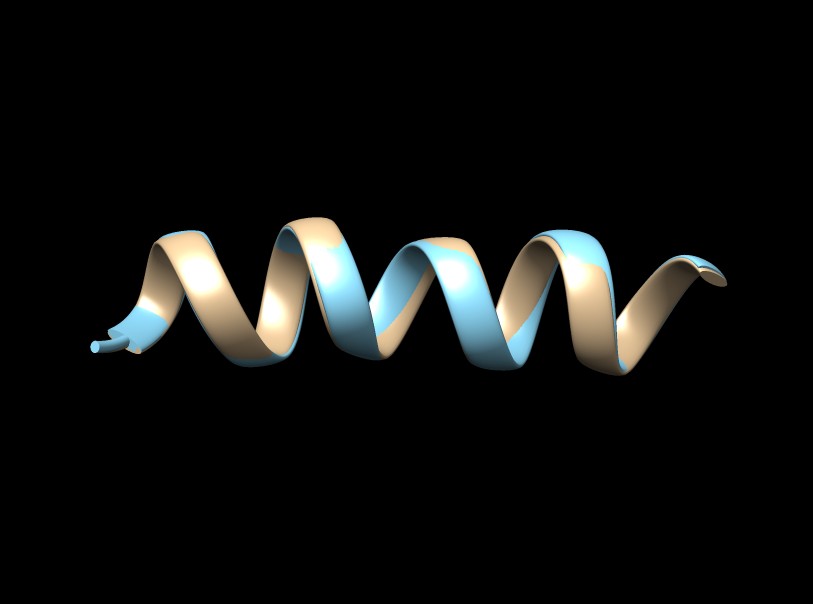}
\label{subfig:norm2MDS}}
\subfigure[Distance Shrinkage, Medium signal-to-noise ratio]{
\includegraphics[scale=0.25]{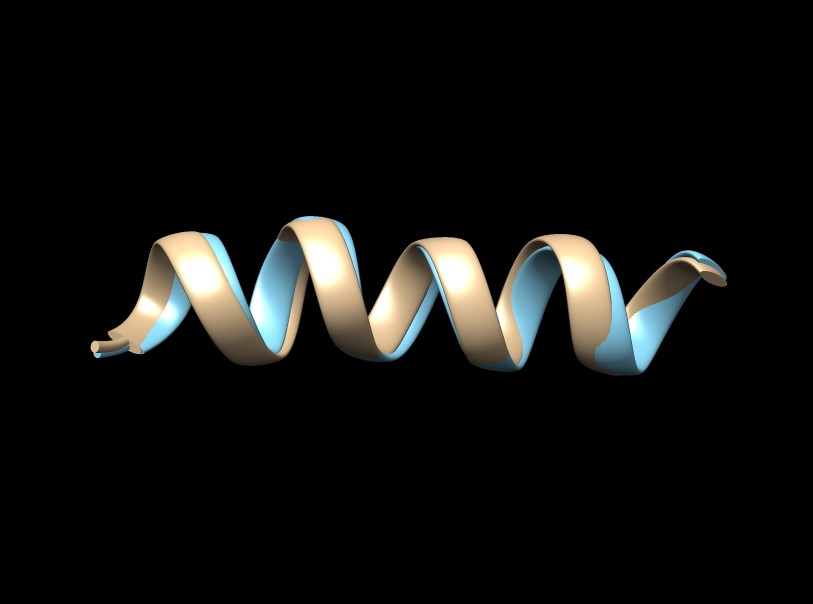}\qquad
\label{subfig:norm12RKE}}
\subfigure[Classical MDS, Medium signal-to-noise ratio]{
\includegraphics[scale=0.25]{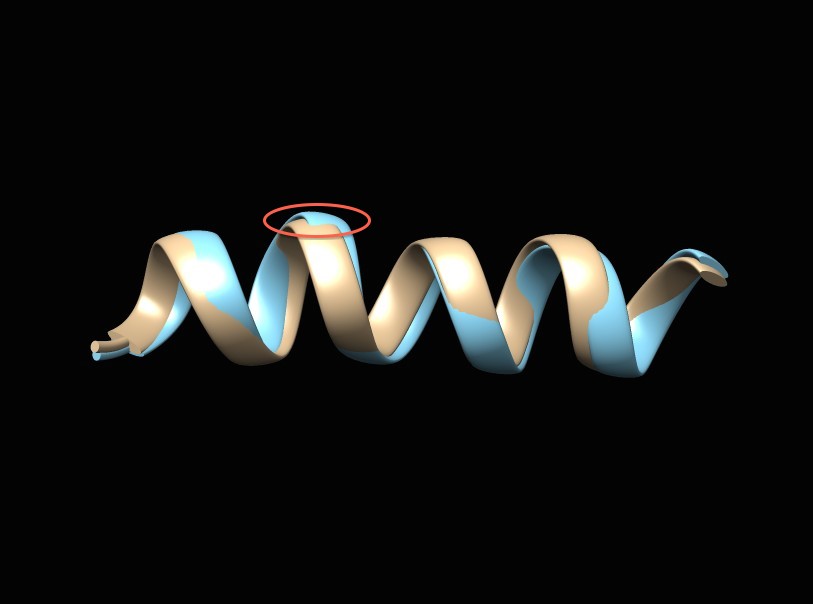}
\label{subfig:norm12MDS}}
\subfigure[Distance Shrinkage, Low signal-to-noise ratio]{
\includegraphics[scale=0.25]{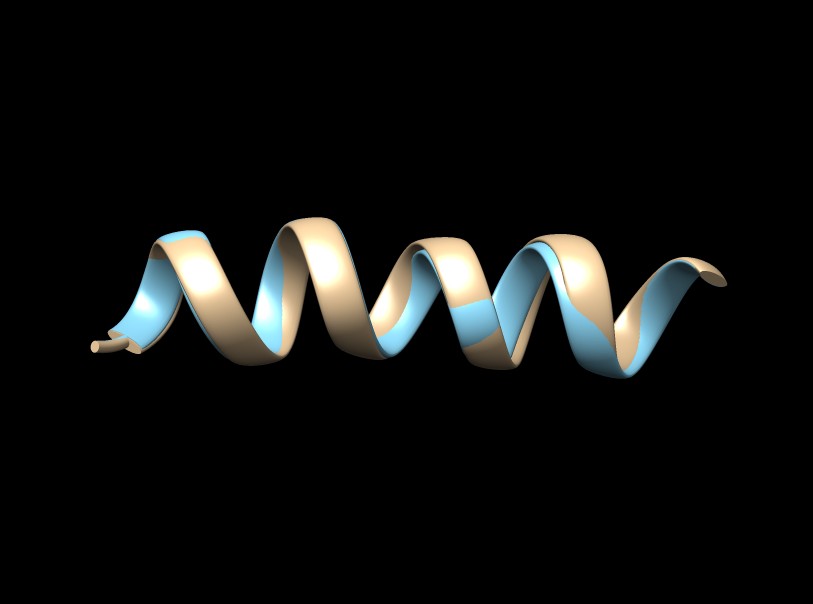}\qquad
\label{subfig:norm20RKE}}
\subfigure[Classical MDS, Low signal-to-noise ratio]{
\includegraphics[scale=0.25]{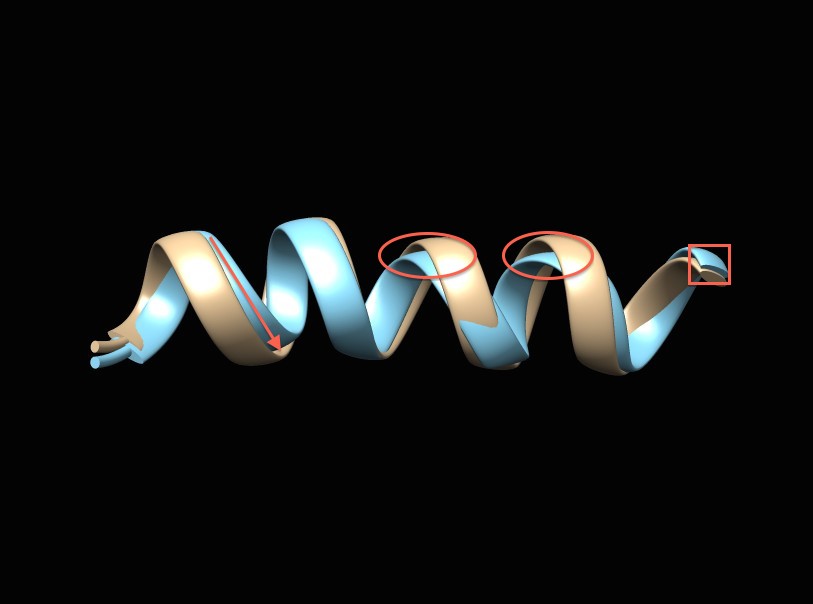}
\label{subfig:norm20MDS}}
\end{center}
\caption{Ribbon plot of 1PJE protein back structure: the true structure is represented in gold whereas the structured corresponding to the estimated Euclidean distance matrix is given in blue. The left panels are for the distance shrinkage estimate whereas the right panels are for the the classical multidimensional scaling. Particular regions where the distance shrinkage shows visible improvement is circled out in red in the right panels.}
\label{fig:1pje}
\end{figure}

Our theoretical analysis suggests better performances for larger number of atoms. To further illustrate this effect of $n$, we repeated the previous experiment for HIV-1 virus protein U cytoplasmic domain (protein data bank symbol: 2K7Y) consisting of $n=671$ atoms. We simulated data in the same fashion as before and the Kruskal stress, based on one hundred simulated dataset for each value of $\sigma^2$, is reported in Table \ref{tab:normdiff}. The performance compares favorable with that for 1PJE.

\begin{table}[htbp]
\begin{center}
\begin{tabular}{c c c c}
\hline \hline
Signal-to-Noise Ratio & Method & mean & standard error \\
\hline
High & Distance Shrinkage & 1.66e-04 & 2.70e-07 \\
\cline{2-4}
 & Classical MDS & 3.2e-03 & 4.84e-06 \\ 
\hline
Medium & Distance Shrinkage & 8.32e-04 & 1.48e-06 \\
\cline{2-4}
 & Classical MDS & 1.61e-02 & 2.45e-05 \\ 
\hline
Low & Distance Shrinkage & 1.7e-03 & 3.05e-06 \\
\cline{2-4}
 & Classical MDS & 3.22e-02 & 5.28e-05 \\ 
\hline
\end{tabular}
\caption{Kruskal's stress for 2K7Y data with measurement error.}
\label{tab:normdiff}
\end{center}
\end{table}

To further demonstrate the robustness of the approach to non-Gaussian measurement error, we generated pairwise distance scores between the 671 atoms following Gamma distributions:
$$
x_{ij}\sim {\rm Ga}(d_{ij},1),\qquad \forall 1\le i<j\le 671,
$$
so that both the mean and variance of $x_{ij}$ are $d_{ij}$, where $d_{ij}$ is the true squared distance between the $i$th and $j$th atoms. We again applied both classical multidimensional scaling and distance shrinkage to estimate the true distance matrix and reconstruct the 3D folding structure. The result from a typical simulated dataset is given in Figure \ref{fig:vpucomp}.

\begin{figure}[htbp]
\begin{center}
\includegraphics[scale=0.5]{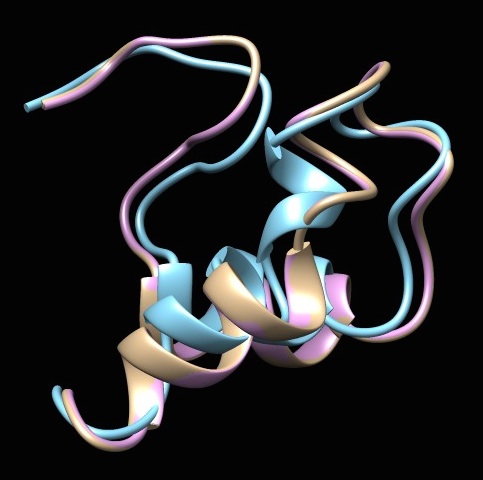}
\end{center}
\caption{Ribbon plot of 2K7Y protein back structure: the true structure, and the structures corresponding to the classical multidimensional scaling and the distance shrinkage estimate are represented in gold, blue and pink respectively.}
\label{fig:vpucomp}
\end{figure}

\section{Proofs}
\label{sec:proof}

\begin{proof}[Proof of Theorem \ref{th:mintr}] Denote by $M_0=-JDJ/2$. We first show that $M_0\in \Mcal(D)$. Note first that
$$
J(e_i-e_j)=(e_i-e_j).
$$
Therefore,
$$
\langle M_0, B_{ij}\rangle=-{1\over 2}(e_i-e_j)^\top JDJ(e_i-e_j)=-{1\over 2}(e_i-e_j)^\top D(e_i-e_j)=d_{ij},
$$
where in the last equality follows from the facts that $D$ is symmetric and $\diag(D)=\zero$. Together with the fact that $M_0\succeq 0$ (Sch\"onberg, 1935; Young and Householder, 1938), this implies that $M_0\in \Mcal(D)$.

Next, we show that for any $M\in \Mcal(D)$, $\tr(M_0)\le \tr(M)$. To this end, observe that
$$
D=\Tscr(M)=\diag(M)\one^\top +\one\diag(M)^\top-2M.
$$
Then
\bes
\tr(M_0)&=&\tr(-JDJ/2)\\
&=&{1\over 2}\tr\left[\left(I-{\one\one^\top\over n}\right)\left(2M-\diag(M)\one^\top -\one\diag(M)^\top\right)\left(I-{\one\one^\top\over n}\right)\right]\\
&=&{1\over 2}\tr\left(2M-\diag(M)\one^\top -\one\diag(M)^\top\right)\\
&&-{1\over n}\one^\top\left(2M-\diag(M)\one^\top -\one\diag(M)^\top\right)\one\\
&&+{1\over 2n^2}\tr\left[\one\one^\top\left(2M-\diag(M)\one^\top -\one\diag(M)^\top\right)\one\one^\top\right]\\
&=&-{1\over 2n}\one^\top\left(2M-\diag(M)\one^\top -\one\diag(M)^\top\right)\one\\
&=&\tr(M)-{1\over n}\one^\top M\one.
\ees
The positive semi-definiteness of $M$ ensures that $\one^\top M\one\ge 0$, which implies that $M_0$ has the minimum trace in $\Mcal(D)$. We now show it is also the only one.

Assume the contrary that there exists an $M\in \Mcal(D)$ such that $M\neq M_0$ yet $\tr(M)=\tr(M_0)$. Following the previous calculation, we have $\one^\top M\one=0$. Recall that $M\succeq 0$. The fact that $\one^\top M\one=0$ necessarily implies that $\one \in {\rm ker}(M)$. As a result, $M=JMJ$, and
$$
M-M_0=J(M-M_0)J.
$$
On the other hand,
$$
\langle M, B_{ij}\rangle=\langle M_0, B_{ij}\rangle=d_{ij}, \quad \forall i<j.
$$
Therefore,
$$
\langle J(M-M_0)J, B_{ij}\rangle=\langle M-M_0, B_{ij}\rangle=0, \quad \forall i<j.
$$
It is not hard to see that
$$
\{B_{ij}: i<j\}\cup \{e_ie_i^\top: 1\le i\le n\}
$$
forms a basis of the collection of $n\times n$ symmetric matrices. In other words, there exists $\alpha_{ij}$ ($1\le i\le j$) such that
$$
M-M_0=\sum_{1\le i< j\le n} \alpha_{ij}B_{ij}+\sum_{i=1}^{n-1} \alpha_{ii}e_ie_i.
$$
Recall that $\one\in {\rm ker}(M)\cap {\rm ker}(M_0)$. Hence
$$
(M-M_0)\one =[\alpha_{11},\ldots,\alpha_{nn}]^\top ={\bf 0}.
$$
In other words,
$$
M-M_0=\sum_{1\le i< j\le n} \alpha_{ij}B_{ij}.
$$
Thus
$$
\|M-M_0\|_{\rm F}^2=\|J(M-M_0)J\|_{\rm F}^2=\sum_{1\le i< j\le n} \alpha_{ij}\langle J(M-M_0)J, B_{ij}\rangle=0.
$$
This obviously contradicts with the assumption that $M\neq M_0$.

The second statement follows from the same argument. Note that $PP^\top\in \Mcal(D)$. Because the embedding points are centered, we have $\one^\top PP^\top \one=0$. The previous argument then suggests that $PP^\top=M_0$.
\end{proof}
\vskip 25pt

\begin{proof}[Proof of Theorem \ref{th:equiv}] Recall that $J=I-(\one\one^\top/n)$. Observe that $D_0=(n-1)I-nJ$. Therefore, for any $M\in \Dcal_n$,
\bes
\left\|\left(X-{\lambda_n\over 2n} D_0\right)-M\right\|_{\rm F}^2&=&\|X-M\|_{\rm F}^2+{\lambda_n\over n}\langle M, D_0\rangle +({\rm terms\ not\ involving\ }M)\\
&=&\|X-M\|_{\rm F}^2+{\lambda_n\over n}\langle M, (n-1)I-nJ\rangle +({\rm terms\ not\ involving\ }M)\\
&=&\|X-M\|_{\rm F}^2-\lambda_n\langle M, J\rangle +({\rm terms\ not\ involving\ }M),
\ees
where the last equality follows from the fact that any distance matrix is hollow, e.g., its diagonals are zeros, hence $\langle M, I\rangle=0$. Because $J$ is idempotent,
$$
\langle M, J\rangle=\langle M, J^2\rangle=\tr(JMJ).
$$
Therefore,
\begin{eqnarray*}
\Pscr_{\Dcal_n}\left(X-{\lambda_n\over 2n} D_0\right)&=&\argmin_{M\in \Dcal_n}\left\{{1\over 2}\|X-M\|_{\rm F}^2-{\lambda_n\over 2} \tr(JMJ)\right\}\\
&=&\argmin_{M\in \Dcal_n}\left\{{1\over 2}\|X-M\|_{\rm F}^2+\lambda_n\tr\left(-{1\over 2}JMJ\right)\right\},
\end{eqnarray*}
which, in the light of (\ref{eq:rkels}), implies the desired statement.
\end{proof}
\vskip 25pt

\begin{proof}[Proof of Theorem \ref{th:oracle}] By Theorem \ref{th:equiv}, $\hat{D}=\Pscr_{\Dcal_n}(X-(\lambda_n/2n)D_0)$. Write $\eta_n=\lambda_n/(2n)$ for simplicity. Recall that for any $M\in \bR^{n\times n}$, its projection to the closed convex set $\Dcal_n$, $\Pscr_{\Dcal_n}(M)$, can be characterized by the so-called Kolmogorov criterion:
$$
\langle A-\Pscr_{\Dcal_n}(M), M-\Pscr_{\Dcal_n}(M)\rangle\le 0,\qquad \forall A\in \Dcal_n.
$$
See, e.g., Escalante and Raydan (2011). In particular, taking $M=X-\eta_n D_0$ yields
$$
\langle A-\hat{D}, D-\hat{D}\rangle \le \langle X-D-\eta_n D_0,\hat{D}-A\rangle.
$$
A classical result in distance geometry by Sch\"onberg (1935) indicates that a distance matrix is conditionally negative semi-definite on the set
$$
\Xcal_n=\{x\in \bR^n: x^\top \one =0\},
$$
that is, $x^\top Mx\le 0$ for any $x\in \Xcal_n$. See also Young and Householder (1938). In other words, if $M\in \Dcal_n$, then the so-called Sch\"onberg transform $JMJ$ is negative semi-definite where, as before, $J=I-(\one\one^\top/n)$.

Let $V$ be the eigenvectors of $JAJ$, and $V_\perp$ be an orthonormal basis of the orthogonal complement of the linear subspace spanned by $\{\one\}$ and $V$. Then $[\one/\sqrt{n}, V,V_\perp]$ forms an orthonormal basis of $\bR^n$. Then for any symmetric matrix $M$, write
$$
M=\Pscr_0 M+\Pscr_1 M,
$$
where
$$
\Pscr_1 M= V_\perp V_\perp^\top MV_\perp V_\perp^\top
$$
and
\bes
\Pscr_0 M&=&M-\Pscr_1 M =[\one/\sqrt{n}, V] [\one/\sqrt{n}, V]^\top M[\one/\sqrt{n}, V] [\one/\sqrt{n}, V]^\top\\
&&\hskip 50pt+V_\perp V_\perp^\top M[\one/\sqrt{n}, V] [\one/\sqrt{n}, V]^\top+[\one/\sqrt{n}, V] [\one/\sqrt{n}, V]^\top MV_\perp V_\perp^\top.
\ees
Therefore,
\bes
\langle X-D, \hat{D}-A\rangle&=&\langle \Pscr_0(X-D), \Pscr_0(\hat{D}-A)\rangle+\langle \Pscr_1(X-D), \Pscr_1(\hat{D}-A)\rangle\\
&=&\langle \Pscr_0(X-D), \Pscr_0(\hat{D}-A)\rangle+\langle \Pscr_1(X-D), \Pscr_1\hat{D}\rangle\\
&\le&\|\Pscr_0(X-D)\|\|\Pscr_0(\hat{D}-A)\|_\ast+\|\Pscr_1(X-D)\|\|\Pscr_1\hat{D}\|_\ast
\ees
where in the last inequality we used the fact that for any matrices $M_1, M_2\in \bR^{n\times n}$,
$$
\langle M_1, M_2\rangle\le \|M_1\|\|M_2\|_\ast,
$$
and $\|\cdot\|$ and $\|\cdot\|_\ast$ represent the matrix spectral and nuclear norm respectively. It is clear that
$$
\|\Pscr_1(X-D)\|\le \|X-D\|,
$$
and
$$
\|\Pscr_0(X-D)\|\le 2\|X-D\|.
$$
Then,
$$
\langle X-D, \hat{D}-A\rangle\le\|X-D\|\left(2\|\Pscr_0(\hat{D}-A)\|_\ast+\|\Pscr_1\hat{D}\|_\ast\right)
$$

On the other hand, recall that both $D$ and $\hat{D}$ are hollow and $D_0=(n-1)I-nJ$. Thus,
\bes
\langle D_0, \hat{D}-A\rangle &=&n\langle A-\hat{D}, J\rangle\\
&=&n\tr(J(A-\hat{D})J)\\
&=&-n\tr(VV^\top(\hat{D}-A)VV^\top)-n\tr(\Pscr_1\hat{D})\\
&=&-n\tr(VV^\top(\hat{D}-A)VV^\top)+n\|\Pscr_1\hat{D}\|_\ast\\
&\ge&-n\|VV^\top(\hat{D}-A)VV^\top\|_\ast+n\|\Pscr_1\hat{D}\|_\ast\\
&\ge&-n\|\Pscr_0(\hat{D}-A)\|_\ast+n\|\Pscr_1\hat{D}\|_\ast,
\ees
where the last equality follows from the fact that $\Pscr_1\hat{D}$ is negative semi-definite.

Taking $n\eta_n \ge \|X-D\|$ yields that
$$
\langle X-D-\lambda_n D_0,\hat{D}-A\rangle\le 3n\eta_n\|\Pscr_0(\hat{D}-A)\|_{\ast}.
$$
Note that, by Cauchy-Schwartz inequality, for any $M\in \bR^{n\times n}$
$$
\|M\|_\ast\le \sqrt{\rank(M)}\|M\|_{\rm F}.
$$
Therefore,
\bes
\|\Pscr_0(\hat{D}-A)\|_{\ast}&\le& \sqrt{\rank(JAJ)+1}\|\Pscr_0(\hat{D}-A)\|_{\rm F}\\
&\le&\sqrt{\rank(JAJ)+1}\|\hat{D}-A\|_{\rm F}\\
&=&\sqrt{\dim(A)+1}\|\hat{D}-A\|_{\rm F},
\ees
where the last equality follows from the fact that for any Euclidean distance matrix $A$, $\dim(A)=\rank(JAJ)$. See, e.g., Sch\"onberg (1935) and Young and Householder (1938). As a result,
$$
\langle A-\hat{D}, D-\hat{D}\rangle \le 3n\eta_n\sqrt{\dim(A)+1}\|\hat{D}-A\|_{\rm F}.
$$

Simple algebraic manipulations show that
$$
\langle A-\hat{D}, D-\hat{D}\rangle={1\over 2}\left(\|\hat{D}-D\|_{\rm F}^2+\|\hat{D}-A\|^2_{\rm F}-\|A-D\|_{\rm F}^2\right).
$$
Thus,
$$
\|\hat{D}-D\|_{\rm F}^2+\|\hat{D}-A\|^2_{\rm F}\le \|A-D\|_{\rm F}^2+6n\eta_n\sqrt{\dim(A)+1}\|\hat{D}-A\|_{\rm F},
$$
which implies that
\bes
\|\hat{D}-D\|_{\rm F}^2&\le&\|A-D\|_{\rm F}^2+6n\eta_n\sqrt{\dim(A)+1}\|\hat{D}-A\|_{\rm F}-\|\hat{D}-A\|^2_{\rm F}\\
&=&\|A-D\|_{\rm F}^2+9n^2\eta_n^2(\dim(A)+1)-\left(\|\hat{D}-A\|_{\rm F}-3n\eta_n\sqrt{\dim(A)+1}\right)^2\\
&\le&\|A-D\|_{\rm F}^2+9n^2\eta_n^2(\dim(A)+1).
\ees
This completes the proof.
\end{proof}
\vskip 25pt

\begin{proof}[Proof of Corollary \ref{co:embed}]
Observe first that
$$
\hat{D}_r=\argmin_{M\in \Dcal_r}\|J(M-\hat{D})J\|_{\rm F}^2.
$$
Therefore,
\bes
\|J(\hat{D}_r-D)J\|_{\rm F}^2&\le& 2\|J(\hat{D}_r-\hat{D})J\|_{\rm F}^2+2\|J(\hat{D}-D)J\|_{\rm F}^2\\
&\le&2\|J(D_r-\hat{D})J\|_{\rm F}^2+2\|\hat{D}-D\|_{\rm F}^2\\
&\le&4\|J(D_r-D)J\|_{\rm F}^2+4\|J(\hat{D}-D)J\|_{\rm F}^2+2\|\hat{D}-D\|_{\rm F}^2\\
&\le&4\min_{M\in \Dcal_n(r)} \|J(D-M)J\|_{\rm F}^2+6\|\hat{D}-D\|_{\rm F}^2
\ees
On the other hand, taking $M=D_r$ in Theorem \ref{th:oracle} yields
\bes
{1\over n^2}\|\widehat{D}-D\|_{\rm F}^2&\le& {1\over n^2}\|D_r-D\|_{\rm F}^2+9\eta_n^2(r+1)\\
&=&{1\over n^2}\min_{M\in \Dcal_n(r)} \|J(D-M)J\|_{\rm F}^2+9\eta_n^2(r+1),
\ees
where, as before, $\eta_n=\lambda_n/2n$. Therefore,
$$
{1\over n^2}\|J(\hat{D}_r-D)J\|_{\rm F}^2\le {10\over n^2}\min_{M\in \Dcal_n(r)} \|J(D-M)J\|_{\rm F}^2+54\eta_n^2(r+1),
$$
which completes the proof.
\end{proof}

\end{document}